\documentclass{article}
\usepackage{amsfonts}

\newtheorem{theorem}{Theorem}

\newtheorem{condition}[theorem]{Condition}

\newtheorem{corollary}[theorem]{Corollary}

\newtheorem{lemma}[theorem]{Lemma}

\newtheorem{proposition}[theorem]{Proposition}

\newenvironment{proof}[1][Proof]{\noindent\textbf{#1.} }{\ \rule{0.5em}{0.5em}}
\input{tcilatex}
\begin{document}

\title{\textbf{Structured Sparsity and Generalization}}
\author{Andreas Maurer \\
Adalbertstr. 55\\
D-80799, M\"{u}nchen\\
\emph{am@andreas-maurer.eu} \and Massimiliano Pontil \\
Dept. of Computer Science, UCL \\
Gower St. London, UK \\
\emph{m.pontil@cs.ucl.ac.uk}}
\maketitle

\begin{abstract}
We present a data dependent generalization bound for a large class of
regularized algorithms which implement structured sparsity constraints. The
bound can be applied to standard squared-norm regularization, the Lasso, the
group Lasso, some versions of the group Lasso with overlapping groups,
multiple kernel learning and other regularization schemes. In all these
cases competitive results are obtained. A novel feature of our bound is that
it can be applied in an infinite dimensional setting such as the Lasso in a
separable Hilbert space or multiple kernel learning with a countable number
of kernels.
\end{abstract}

\section{Introduction}

We study a class of regularization methods used to learn a linear function
from a finite set of examples. The regularizer is expressed as an infimum
convolution which involves a set $\mathcal{M}$ of linear transformations
(see equation \eqref{def Omega} below). As we shall see, this regularizer
generalizes, depending on the choice of the set $\mathcal{M}$, the
regularizers used by several learning algorithms, such as ridge regression,
the Lasso, the group Lasso \cite{YuaLin}, multiple kernel learning \cite%
{Lanckriet}, the group Lasso with overlap \cite{Obozinski 2009}, and the
regularizers in \cite{Micchelli 2010}.

We give a bound on the Rademacher average of the linear function class
associated with this regularizer. The result matches existing bounds in the
above mentioned cases but also admits a novel, dimension free
interpretation. In particular, the bound applies to the Lasso in $\ell _{2}$
or to multiple kernel learning with a countable number of kernels, under
certain finite second-moment conditions. \bigskip

Let $H$ be a real Hilbert space with inner product $\langle \cdot ,\cdot
\rangle $ and induced norm $\Vert \cdot \Vert $. Let $\mathcal{M}$ be a
finite or countably infinite set of symmetric bounded linear operators on $H$
such that for every $x\in H$, $x\neq 0,$ there is some linear operator $M\in 
\mathcal{M}$ with $Mx\neq 0$ and that $\sup_{M\in \mathcal{M}}|||M|||<\infty 
$, where $|||\cdot |||$ is the operator norm. Define the function $%
\left\Vert \cdot \right\Vert _{\mathcal{M}}:H\rightarrow 
\mathbb{R}
_{+}\cup \left\{ \infty \right\} $ by%
\begin{equation}
\left\Vert \beta \right\Vert _{\mathcal{M}}=\inf \left\{ \sum_{M\in \mathcal{%
M}}\left\Vert v_{M}\right\Vert :v_{M}\in H,~\sum_{M\in \mathcal{M}%
}Mv_{M}=\beta \right\} .  \label{def Omega}
\end{equation}%
It is shown in Section \ref{section Norm Properties} that the chosen
notation is justified, because $\left\Vert \cdot \right\Vert _{\mathcal{M}}$
is indeed a norm on the subspace of $H$ where it is finite, and the dual
norm is, for every $z\in H$, given by 
\[
\left\Vert z\right\Vert _{\mathcal{M}\ast }=\sup_{M\in \mathcal{M}%
}\left\Vert Mz\right\Vert .
\]%
The somewhat complicated definition of~$\left\Vert \cdot \right\Vert _{%
\mathcal{M}}$ is contrasted by the simple form of the dual norm.\bigskip 

Using well known techniques, as described in \cite{Koltchinskii 2002} and 
\cite{Bartlett 2002}, our study of generalization reduces to the search for
a good bound on the empirical Rademacher complexity of a set of linear
functionals with $\left\Vert \cdot\right\Vert _{\mathcal{M}}$-bounded weight
vectors%
\begin{equation}
\mathcal{R}_{\mathcal{M}}\left( \mathbf{x}\right) =\frac{2}{n}{{\mathbb{E}}}%
\sup_{\beta :~\left\Vert \beta \right\Vert _{\mathcal{M}}\leq
1}\sum_{i=1}^{n}\epsilon _{i}\left\langle \beta ,x_{i}\right\rangle ,
\label{eq Rademacher complexity}
\end{equation}%
where $\mathbf{x}=\left( x_{1},\dots ,x_{n}\right) \in H^{n}$ is a sample
vector representing observations, and $\epsilon _{1},\dots ,\epsilon _{n}$
are Rademacher variables, mutually independent and each uniformly
distributed on $\left\{ -1,1\right\} $\footnote{%
Our definition coincides with the one in \cite{Bartlett 2002}, while other
authors omit the factor of $2$. This is relevant when comparing the
constants in different bounds.}. Given a bound on $\mathcal{R}_{\mathcal{M}%
}\left( \mathbf{x}\right) $ we obtain uniform bounds on the estimation
error, for example using the following standard result (adapted from \cite%
{Bartlett 2002}), where the Lipschitz function $\phi $ is to be interpreted
as a loss function.

\begin{theorem}
Let $\mathbf{X}=\left( X_{1},\dots ,X_{n}\right) $ be a vector of iid random
variables with values in $H$, let $X$ be iid to $X_{1}$, let $\phi :%
\mathbb{R}
\rightarrow \left[ 0,1\right] $ have Lipschitz constant $L$ and $\delta \in
(0,1)$. Then with probability at least $1-\delta $ in the draw of $\mathbf{X}
$ it holds, for every $\beta \in \mathbb{R}^{d}$ with $\left\Vert \beta
\right\Vert _{\mathcal{M}}\leq 1$, that 
\[
{{\mathbb{E}}}\phi \left( \left\langle \beta ,X\right\rangle \right) \leq 
\frac{1}{n}\sum_{i=1}^{n}\phi \left( \left\langle \beta ,X_{i}\right\rangle
\right) +L~\mathcal{R}_{\mathcal{M}}\left( \mathbf{X}\right) +\sqrt{\frac{%
9\ln 2/\delta }{2n}}.
\]
\end{theorem}

A similar (slightly better) bound is obtained if $\mathcal{R}_{\mathcal{M}%
}\left( \mathbf{X}\right) $ is replaced by its expectation $\mathcal{R}_{%
\mathcal{M}}={{\mathbb{E}}}\mathcal{R}_{\mathcal{M}}\left( \mathbf{X}\right) 
$ (see \cite{Bartlett 2002}).

The following is the main result of this paper and leads to consistency
proofs and finite sample generalization guarantees for all algorithms which
use a regularizer of the form (\ref{def Omega}). A proof is given in Section %
\ref{Section Proofs of bounds}.\bigskip

\begin{theorem}
\label{Theorem Main} Let $\mathbf{x}=\left( x_{1},\dots ,x_{n}\right) \in
H^{n}$ and $\mathcal{R}_{\mathcal{M}}\left( \mathbf{x}\right) $ be defined
as in (\ref{eq Rademacher complexity}). Then%
\begin{eqnarray*}
\mathcal{R}_{\mathcal{M}}\left( \mathbf{x}\right)  &\leq &\frac{2^{3/2}}{n}%
\sqrt{\sup_{M\in \mathcal{M}}\sum_{i=1}^{n}\left\Vert Mx_{i}\right\Vert ^{2}}%
\left( 2+\sqrt{\ln \left( \sum\limits_{M\in \mathcal{M}}\frac{%
\sum_{i}\left\Vert Mx_{i}\right\Vert ^{2}}{\sup\limits_{N\in \mathcal{M}%
}\sum_{j}\left\Vert Nx_{j}\right\Vert ^{2}}\right) }\hspace{0.05cm}\right) 
\\
&\leq &\frac{2^{3/2}}{n}\sqrt{\sum_{i=1}^{n}\left\Vert x_{i}\right\Vert _{%
\mathcal{M}\ast }^{2}}\left( 2+\sqrt{\ln \left\vert \mathcal{M}\right\vert }%
\right) .
\end{eqnarray*}
\end{theorem}

The second inequality follows from the first, the inequality 
\[
\sup_{M\in \mathcal{M}}\sum_{i=1}^{n}\left\Vert Mx_{i}\right\Vert ^{2}\leq
\sum_{i=1}^{n}\left\Vert x_{i}\right\Vert _{\mathcal{M}\ast }^{2}
\]%
(a fact which will be tacitly used in the sequel) and the observation that
every summand in the logarithm appearing in the first inequality is bounded
by $1$. Of course the second inequality is relevant only if $\mathcal{M}$ is
finite. In this case we can draw the following conclusion: If we have an a
priori bound on $\left\Vert X\right\Vert _{\mathcal{M}\ast }$ for some data
distribution, say $\left\Vert X\right\Vert _{\mathcal{M}\ast }\leq C$, and $%
\mathbf{X}=\left( X_{1},\dots ,X_{n}\right) $, with $X_{i}$ iid to $X$, then 
\[
\mathcal{R}_{\mathcal{M}}\left( \mathbf{X}\right) \leq \frac{2^{3/2}C}{\sqrt{%
n}}\left( 2+\sqrt{\ln \left\vert \mathcal{M}\right\vert }\right) ,
\]%
thus passing from a data-dependent to a distribution dependent bound. In
Section \ref{sec:examples} we show that this recovers existing results for
many regularization schemes.

But the first bound in Theorem \ref{Theorem Main} can be considerably
smaller than the second and may be finite even if $\mathcal{M}$ is infinite.
This gives rise to some appearantly novel features, even in the well studied
case of the Lasso, when there is a (finite but potentially large) $\ell _{2}$%
-bound on the data.

\begin{corollary}
\label{Theorem L2 bound}Under the conditions of Theorem \ref{Theorem Main}
we have%
\[
\mathcal{R}_{\mathcal{M}}\left( \mathbf{x}\right) \leq \frac{2^{3/2}}{n}%
\sqrt{\sup_{M\in \mathcal{M}}\sum_{i}\left\Vert Mx_{i}\right\Vert ^{2}}%
\left( 2+\sqrt{\ln \frac{1}{n}\sum_{i}\sum_{M\in \mathcal{M}}\left\Vert
Mx_{i}\right\Vert ^{2}}\right) +\frac{2}{\sqrt{n}}. 
\]
\end{corollary}

A proof is given in Section \ref{Section Proofs of bounds}. To obtain a
novel distribution dependent bound we retain the condition $\left\Vert
X\right\Vert _{\mathcal{M}\ast }\leq C$ and replace finiteness of $\mathcal{M%
}$ by the condition that%
\begin{equation}
R^{2}:={{\mathbb{E}}}\sum_{M\in \mathcal{M}}\left\Vert MX\right\Vert
^{2}<\infty .  \label{eq second moment condition}
\end{equation}%
Taking the expectation in Corollary \ref{Theorem L2 bound} then gives a
bound on the expected Rademacher complexity%
\begin{equation}
\mathcal{R}_{\mathcal{M}}\leq \frac{2^{3/2}C}{\sqrt{n}}\left( 2+\sqrt{\ln
R^{2}}\right) +\frac{2}{\sqrt{n}}.  \label{eq second moment bound}
\end{equation}%
The key features of this result are the dimension-independence and the only
logarithmic dependence on $R^{2}$, which in many applications turns out to
be simply $R^{2}={{\mathbb{E}}}\left\Vert X\right\Vert ^{2}$.

\bigskip
The rest of the paper is organized as follows. In the next section we specialize our results to different regularizers. In Section \ref{sec:proofs} we present the proof of Theorem \ref{Theorem Main} as well as the proof of other results mentioned above. In Section \ref{sec:Lp} we discuss the extension of these results to the $\ell_q$ case. Finally we draw our conclusions and comment on future work.

\section{Examples}

\label{sec:examples} Before giving the examples we mention a great
simplification in the definition of the norm $\left\Vert \cdot \right\Vert _{%
\mathcal{M}}$ which occurs when the members of $\mathcal{M}$ have mutually
orthogonal ranges. A simple argument, given in Proposition \ref{Proposition
simplification} below shows that in this case 
\[
\left\Vert \beta \right\Vert _{\mathcal{M}}=\sum_{M\in \mathcal{M}%
}\left\Vert M^{+}\beta \right\Vert ,
\]
where $M^{+}$ is the pseudoinverse of $M$. If,\textit{\ in addition}, every
member of $\mathcal{M}$ is an orthogonal projection $P$, the norm further
simplifies to%
\[
\left\Vert \beta \right\Vert _{\mathcal{M}}=\sum_{P\in \mathcal{M}%
}\left\Vert P\beta \right\Vert ,
\]%
and the quantity $R^{2}$ occurring in the second moment condition (\ref{eq
second moment condition}) simplifies to%
\[
R^{2}={{\mathbb{E}}}\sum_{P\in \mathcal{M}}\left\Vert PX\right\Vert ^{2}={{%
\mathbb{E}}}\left\Vert X\right\Vert ^{2}.
\]

For the remainder of this section $\mathbf{X}=\left( X_{1},\dots
,X_{n}\right) $ will be a generic iid random vector of data points, $%
X_{i}\in H$, and $X$ will be a generic data variable, iid to $X_{i}$. If $H=%
\mathbb{R}
^{d}$ we write $\left( X\right) _{k}$ for the $k$-th coordinate of $X$, not
to be confused with $X_{k}$, which would be the $k$-th member of the vector $%
\mathbf{X}$.

\subsection{The Euclidean Regularizer}

In this simplest case we set $\mathcal{M}=\left\{ I\right\} $, where $I$ is
the identity operator on the Hilbert space $H$. Then $\left\Vert \beta
\right\Vert _{\mathcal{M}}=\left\Vert \beta \right\Vert$, $\left\Vert
z\right\Vert _{\mathcal{M\ast }}=\left\Vert z\right\Vert$, and the bound on
the empirical Rademacher complexity becomes%
\[
\mathcal{R}_{\mathcal{M}}\left( \mathbf{x}\right) \leq \frac{2^{5/2}}{n}%
\sqrt{\sum_{i}\left\Vert x_{i}\right\Vert ^{2}}, 
\]%
worse by a constant factor of $2^{3/2}$ than the corresponding result in 
\cite{Bartlett 2002}, a tribute paid to the generality of our result.

\subsection{The Lasso}

Let us first assume that $H=%
\mathbb{R}
^{d}$ is finite dimensional and set $\mathcal{M}=\left\{ P_{1},\dots
,P_{d}\right\} $ where $P_{k}$ is the orthogonal projection onto the $1$%
-dimensional subspace generated by the basis vector $e_{k}$. All the above
mentioned simplifications apply and we have $\left\Vert \beta \right\Vert _{%
\mathcal{M}}=\left\Vert \beta \right\Vert _{1}$ and $\left\Vert z\right\Vert
_{\mathcal{M\ast }}=\left\Vert z\right\Vert _{\infty }$. Writing $x_{ik}$
for the $k$-th coordinate of a data-point $x_{i}$, the bound on $\mathcal{R}%
_{\mathcal{M}}\left( \mathbf{x}\right) $ now reads 
\[
\mathcal{R}_{\mathcal{M}}\left( \mathbf{x}\right) \leq \frac{2^{3/2}}{n}%
\sqrt{\sum_{i}\left\Vert x_{i}\right\Vert _{\mathcal{\infty }}^{2}}\left( 2+%
\sqrt{\ln d}\right) . 
\]%
If $\left\Vert X\right\Vert _{\infty }\leq 1$ almost surely we obtain%
\[
\mathcal{R}_{\mathcal{M}}\left( \mathbf{X}\right) \leq \frac{2^{3/2}}{\sqrt{n%
}}\left( 2+\sqrt{\ln d}\right) , 
\]%
which agrees with the bound in \cite{Kakade} on the dominant term (see also 
\cite{Bartlett 2002},\cite{Zhang}).

Our last bound is useless if $d\geq e^{n}$ or if $d$ is infinite. But
whenever the norm of the data has finite second moments we can use Corollary %
\ref{Theorem L2 bound} and (\ref{eq second moment bound}) to obtain%
\[
\mathcal{R}_{\mathcal{M}}\left( \mathbf{X}\right) \leq \frac{2^{3/2}}{\sqrt{n%
}}\left( 2+\sqrt{\ln {{\mathbb{E}}}\left\Vert X\right\Vert _{2}^{2}}\right) +%
\frac{2}{\sqrt{n}}. 
\]%
For nontrivial results ${{\mathbb{E}}}\left\Vert X\right\Vert ^{2}$ only
needs to be subexponential in $n$.

We remark that a similar condition to equation \eqref{eq  second moment condition} for the Lasso, replacing the expectation with the supremum over $X$, has been considered within the context of elastic net regularization \cite{DDR}. 

\subsection{The Weighted Lasso}

The Lasso assigns an equal penalty to all regression coefficients, while
there may be a priori information on the respective significance of the
different coordinates. For this reason different weightings have been
proposed (see e.g. \cite{shimamura}). In our framework an appropriate set of
operators is $\mathcal{M}=\left\{ \alpha _{1}P_{1},\dots ,\alpha
_{k}P_{k},\dots \right\} $, with $\alpha _{k}>0$ where $\alpha _{k}^{-1}$ is
the penalty weight associated with the $k$-th coordinate. Then 
\[
\left\Vert \beta \right\Vert _{\mathcal{M}}=\sum_{k}\alpha
_{k}^{-1}\left\vert \beta _{k}\right\vert 
\]%
and 
\[
\left\Vert z\right\Vert _{\mathcal{M\ast }}=\sup_{k}\alpha _{k}\left\vert
z_{k}\right\vert \text{.} 
\]%
To further illustrate the use of Corollary \ref{Theorem L2 bound} let us
assume that the underlying space $H$ is infinite dimensional (i.e. $H=\ell
_{2}\left( 
\mathbb{N}
\right) $), and make the compensating assumption that $\alpha \in H$, i.e. $%
\sum_{k}\alpha _{k}^{2}=R^{2}<\infty $. For simplicity we also assume that $%
\sup_{k}\alpha _{k}\leq 1$. Then, if $\left\Vert X\right\Vert _{\infty }\leq
1$ almost surely, we have both $\left\Vert X\right\Vert _{\mathcal{M\ast }%
}\leq 1$ and $\sum_{k}\alpha _{k}^{2}\left( X\right) _{k}^{2}\leq R^{2}$.
Again we obtain 
\[
\mathcal{R}_{\mathcal{M}}\left( \mathbf{X}\right) \leq \frac{2^{3/2}}{\sqrt{n%
}}\left( 2+\sqrt{\ln R^{2}}\right) +\frac{2}{\sqrt{n}}. 
\]%
So in this case the second moment bound is enforced by the weighting
sequence.

\subsection{The Group Lasso}

Let $H=%
\mathbb{R}
^{d}$ and let $\left\{ J_{1},\dots ,J_{r}\right\} $ be a partition of the
index set $\left\{ 1,\dots ,d\right\} $. We take $\mathcal{M=}\left\{
P_{J_{1}},\dots ,P_{J_{r}}\right\} $ where $P_{J_{\ell }}=\sum_{i\in J_{\ell
}}P_{i}$ is the projection onto the subspace spanned by the basis vector $%
e_{i}$. The ranges of the $P_{J_{\ell }}$ then provide an orthogonal
decomposition of $%
\mathbb{R}
^{d}$ and the above mentioned simplifications also apply. We get%
\[
\left\Vert \beta \right\Vert _{\mathcal{M}}=\sum_{\ell =1}^{r}\left\Vert
P_{J_{\ell }}\beta \right\Vert 
\]%
and%
\[
\left\Vert z\right\Vert _{\mathcal{M\ast }}=\max_{\ell =1}^{r}\left\Vert
P_{J_{\ell }}z\right\Vert .
\]%
The algorithm which uses $\left\Vert \beta \right\Vert _{\mathcal{M}}$ as a
regularizer is called the group Lasso (see e.g. \cite{YuaLin}). It
encourages vectors $\beta $ whose support lies the union of a small number
of groups $J_{\ell }$ of coordinate indices. If we know that $\left\Vert
P_{J_{\ell }}X\right\Vert \leq 1$ almost surely for all $\ell \in \left\{
1,\dots ,r\right\} $ then we get%
\begin{equation}
\mathcal{R}_{\mathcal{M}}\left( \mathbf{X}\right) \leq \frac{2^{3/2}}{\sqrt{n%
}}\left( 2+\sqrt{\ln r}\right) ,  \label{eq group Lasso bound}
\end{equation}%
in complete symmetry with the Lasso and essentially the same as given in 
\cite{Kakade}. If $r$ is prohibitively large or if different penalties are
desired for different groups, the same remarks apply as in the previous two
sections. Just as in the case of the Lasso the second moment condition (\ref%
{eq second moment condition}) translates to the simple form ${{\mathbb{E}}}%
\left\Vert X\right\Vert _{2}^{2}<\infty $.

\subsection{Overlapping Groups}

In the previous examples the members of $\mathcal{M}$ always had mutually
orthogonal ranges, which gave a simple appearance to the norm $\left\Vert
\beta \right\Vert _{\mathcal{M}}$. If the ranges are not mutually
orthogonal, the norm has a more complicated form. For example, in the group
Lasso setting, if the groups $J_{\ell }$ cover $\left\{ 1,\dots ,d\right\} $%
, but are not disjoint, we obtain the regularizer of \cite{Obozinski 2009},
given by 
\[
\Omega _{\mathrm{overlap}}\left( \beta \right) =\inf \left\{ \sum_{\ell
=1}^{r}\left\Vert v_{\ell }\right\Vert :(v_{\ell })_{jk}=0~\text{if}~k\notin
J_{\ell }~\mathrm{and}~\sum_{\ell =1}^{r}v_{\ell }=\beta \right\} . 
\]%
If $\left\Vert P_{J_{\ell }}X_{i}\right\Vert \leq 1$ almost surely for all $%
\ell \in \left\{ 1,\dots ,r\right\} $ then the Rademacher complexity of the
set of linear functionals with $\Omega _{\mathrm{overlap}}\left( \beta
\right) \leq 1$ is bounded as in (\ref{eq group Lasso bound}), in complete
equivalence to the bound for the group Lasso.

The same bound also holds for the class satisfying $\Omega _{\mathrm{group}%
}\left( \beta \right) \leq 1$, where the function $\Omega _{\mathrm{group}}$
is defined, for every $\beta \in 
\mathbb{R}
^{d}$, as%
\[
\Omega _{\mathrm{group}}\left( \beta \right) =\sum_{\ell =1}^{r}\left\Vert
P_{J_{\ell }}\beta \right\Vert 
\]%
which has been proposed by \cite{Jenatton,Zhao}. To see this we only have to
show that $\Omega _{\mathrm{overlap}}\leq \Omega _{\mathrm{group}}$ which is
accomplished by generating a disjoint partition $\left\{ J_{\ell }^{\prime
}\right\} _{\ell =1}^{r}$ where $J_{\ell }^{\prime }\subseteq J_{\ell }$,
writing $\beta =\sum_{\ell =1}^{r}P_{J_{\ell }^{\prime }}\beta $ and
realizing that $\left\Vert P_{J_{\ell }^{\prime }}\beta \right\Vert \leq
\left\Vert P_{J_{\ell }}\beta \right\Vert $. The bound obtained from this
simple comparison may however be quite loose.

\subsection{Regularizers Generated from Cones}

Our next example considers structured sparsity regularizers as in \cite%
{Micchelli 2010}. Let $\Lambda $ be a nonempty subset of the open positive
orthant in $%
\mathbb{R}
^{d}$ and define a function $\Omega _{\Lambda }:%
\mathbb{R}
^{d}\rightarrow 
\mathbb{R}
$ by 
\[
\Omega _{\Lambda }\left( \beta \right) =\frac{1}{2}\inf_{\lambda \in \Lambda
}\sum_{j=1}^{d}\left( \frac{\beta _{j}^{2}}{\lambda _{j}}+\lambda
_{j}\right) . 
\]%
If $\Lambda $ is a convex cone, then it is shown in \cite{Micchelli 2011}
that $\Omega _{\Lambda }$ is a norm and that the dual norm is given by 
\[
\left\Vert z\right\Vert _{\Lambda \ast }=\sup \left\{ \left(
\sum_{j=1}^{d}\mu _{j}z_{j}^{2}\right) ^{1/2}:\mu _{j}=\lambda /\left\Vert
\lambda \right\Vert _{1}\text{ with }\lambda \in \Lambda \right\} . 
\]%
The supremum in this formula is evidently attained on the set $\mathcal{E}%
\left( \Lambda \right) $ of extreme points of the closure of $\left\{
\lambda /\left\Vert \lambda \right\Vert _{1}:\lambda \in \Lambda \right\} $.
For $\mu \in \mathcal{E}\left( \Lambda \right) $ let $M_{\mu }$ be the
diagonal matrix with entries $\sqrt{\mu _{j}}\delta _{jl}$ and let $\mathcal{%
M}_{\Lambda }$ be the collection of matrices $\mathcal{M}_{\Lambda }=\left\{
M_{\mu }:\mu \in \mathcal{E}\left( \Lambda \right) \right\} $. Then%
\[
\left\Vert z\right\Vert _{\Lambda \ast }=\sup_{M\in \mathcal{M}_{\Lambda
}}\left\Vert Mz\right\Vert . 
\]%
Clearly $\mathcal{M}_{\Lambda }$ is uniformly bounded in the operator norm,
so if $\Lambda $ is a cone and $\mathcal{E}\left( \Lambda \right) $ is at
most countable, then $\left\Vert \cdot \right\Vert _{\Lambda \ast
}=\left\Vert \cdot\right\Vert _{\mathcal{M}\ast }$, $\Omega _{\Lambda
}=\left\Vert \cdot\right\Vert _{\mathcal{M}\ast }$ and our bounds apply. If $%
\mathcal{E}\left( \Lambda \right) $ is finite and $\mathbf{x}$ is a sample
then the Rademacher complexity of the class with $\Omega _{\Lambda }\left(
\beta \right) \leq 1$ is bounded by 
\[
\frac{2^{3/2}}{n}\sqrt{\sum_{i}\left\Vert x_{i}\right\Vert _{\Lambda \ast
}^{2}}\left( 2+\sqrt{\ln \left\vert \mathcal{E}\left( \Lambda \right)
\right\vert }\right) . 
\]

\subsection{Kernel Learning}

This is the most general case to which the simplification applies: Suppose
that $H$ is the direct sum $H=\oplus _{{j}\in {\mathcal{J}}}H_{j}$ of an at
most countable number of 
Hilbert spaces $H_{j}$. 
We set $\mathcal{M}=\left\{ P_{{j}}\right\} _{{j} \in {\mathcal{J}}}$, where 
$P_{j}: H \rightarrow H$ is the projection on $H_{j}$. Then 
\[
\left\Vert \beta \right\Vert _{\mathcal{M}}=\sum_{{j}\in {\mathcal{J}}%
}\left\Vert P_{{j}}\beta \right\Vert 
\]
and 
\[
\left\Vert z\right\Vert _{\mathcal{M\ast }}=\sup_{{j}\in {\mathcal{J}}%
}\left\Vert P_{{j}}z\right\Vert . 
\]
Such a situation arises in multiple kernel learning \cite{Lanckriet} or the
nonparametric group Lasso \cite{Sara} in the following way: One has an input
space $\mathcal{X}$ and a collection $\left\{ K_{{j}}\right\} _{{j}\in {%
\mathcal{J}}}$ of positive definite kernels $K_{{j}}:\mathcal{X\times
X\rightarrow \mathbb{R} }$. Let $\phi_j: {\mathcal{X}} \rightarrow H_j$ be
the feature map representation associated with kernel $K_j$, so that, for
every $x,t \in {\mathcal{X}}$ $K_j(x,t) = \langle \phi_j(x),\phi_j(t)\rangle$
(for background on kernel methods see, for example, \cite{ST}). 

Suppose that $\mathbf{x}=\left( x_{1},\dots ,x_{n}\right) \in \mathcal{X}%
^{n} $ is a sample. Define the kernel matrix $\mathbf{K}_{{j}}=(K_{{j}%
}(x_{i},x_{k}))_{i,k=1}^{n}$. Using this notation the bound in Theorem \ref%
{Theorem Main} reads 
\[
\mathcal{R}\left( (\phi (x_{1}),\dots ,\phi (x_{n}))\right) \leq \frac{%
2^{3/2}}{n}\sqrt{\sup_{{j}\in {\mathcal{J}}}\mathrm{tr}\mathbf{K}_{{j}}}%
\left( 2+\sqrt{\ln \frac{\sum_{{j}\in {\mathcal{J}}}\mathrm{tr}\mathbf{K}_{{j%
}}}{\sup_{{j}\in {\mathcal{J}}}\mathrm{tr}\mathbf{K}_{{j}}}}\right).
\]%
In particular, if ${\mathcal{J}}$ is finite and $K_{{j}}(x,x)\leq 1$ for every $x\in 
\mathcal{X}$ and ${j}\in {\mathcal{J}}$, then the the bound reduces to 
\[
\frac{2^{3/2}}{\sqrt{n}}\left( 2+\sqrt{\ln \left\vert {\mathcal{J}}%
\right\vert }\right) , 
\]%
essentially in agreement with \cite{Cortes,Kakade,Yiming}. For infinite or
prohibitively large ${\mathcal{J}}$ the second moment condition now becomes 
\[
\mathbb{E}\sum_{{j}\in {\mathcal{J}}}K_{{j}}\left( X,X\right) <\infty \text{.%
} 
\]

\section{Proofs}
\label{sec:proofs}
We first give some notation and auxiliary results, then we prove the results
announced in the introduction.

\subsection{Notation and Auxiliary Results}

The Hilbert space $H$ and the collection $M$ are fixed throughout the
following, as is the sample size $n\in 
\mathbb{N}
$.

Recall that $\left\Vert \cdot \right\Vert $ and $\left\langle \cdot ,\cdot
\right\rangle $ denote the norm and inner product in $H$, respectively. For
a linear transformation $M:%
\mathbb{R}
^{n}\rightarrow H$ the Hilbert-Schmidt norm is defined as%
\[
\left\Vert M\right\Vert _{HS}=\left( \sum_{i=1}^{n}\left\Vert
Me_{i}\right\Vert ^{2}\right) ^{1/2}
\]%
where $\{e_{i}:i\in {\mathbb{N}}\}$ is the canonical basis of $%
\mathbb{R}
^{n}$.

We use bold letters ($\mathbf{x}$, $\mathbf{X}$, $\mathbf{\epsilon }$,
\dots) to denote $n$-tuples of objects, such as vectors or random variables.

Let $\mathcal{X}$ be any space. For $\mathbf{x}=\left(
x_{1},\dots,x_{n}\right) \in \mathcal{X}^{n}$, $1\leq k\leq n$ and $y\in 
\mathcal{X}$ we use $\mathbf{x}_{k\leftarrow y}$ to denote the object
obtained from $\mathbf{x}$ by replacing the $k$-th coordinate of $\mathbf{x}$
with $y$. That is 
\[
\mathbf{x}_{k\leftarrow y}=\left(
x_{1},\dots,x_{k-1},y,x_{k+1},\dots,x_{n}\right) \text{.} 
\]%
The following concentration inequality, known as the bounded difference
inequality (see McDiarmid \cite{McDiarmid 1998}), goes back to the work of
Hoeffding \cite{Hoeffding 1963}. We only need it in the weak form stated
below.

\begin{theorem}
\label{Theorem Bded Difference}Let $F:\mathcal{X}^{n}\rightarrow 
\mathbb{R}
$ and write%
\[
B^{2}=\sum_{k=1}^{n}\sup_{y_{1},y_{2}\in \mathcal{X}\text{, }\mathbf{x}\in 
\mathcal{X}^{n}}\left( F\left( \mathbf{x}_{k\leftarrow y_{1}}\right)
-F\left( \mathbf{x}_{k\leftarrow y_{2}}\right) \right) ^{2}. 
\]%
Let $\mathbf{X}=\left( X_{1},\dots,X_{n}\right) $ be a vector of independent
random variables with values in $\mathcal{X}$, and let $\mathbf{X}^{\prime }$
be iid to $\mathbf{X}$. Then for any $t>0$%
\[
\Pr \left\{ F\left( \mathbf{X}\right) >{{\mathbb{E}}} F\left( \mathbf{X}%
^{\prime }\right) +t\right\} \leq e^{-2t^{2}/B^{2}}. 
\]%
\bigskip
\end{theorem}

Finally we need a simple lemma on the normal approximation:

\begin{lemma}
\label{Lemma Normal approximation} Let $a,\delta >0$. Then 
\[
\int_{\delta }^{\infty }\exp \left( \frac{-t^{2}}{2a^{2}}\right) dt\leq 
\frac{a^{2}}{\delta }\exp \left( \frac{-\delta ^{2}}{2a^{2}}\right) . 
\]
\end{lemma}

\begin{proof}
For $t\geq \delta /a$ we have $1\leq at/\delta $. Thus 
\[
\int_{\delta }^{\infty }\exp \left( \frac{-t^{2}}{2a^{2}}\right)
dt=a\int_{\delta /a}^{\infty }e^{-t^{2}/2}dt\leq \frac{a^{2}}{\delta }%
\int_{\delta /a}^{\infty }te^{-t^{2}/2}dt=\frac{a^{2}}{\delta }\exp \left( 
\frac{-\delta ^{2}}{2a^{2}}\right) . 
\]
\end{proof}

\subsection{Properties of $\left\Vert \cdot\right\Vert _{\mathcal{M}}$ and
Duality\label{section Norm Properties}}

We state again the general conditions on the set $\mathcal{M}$.\bigskip 
\begin{condition}
$\mathcal{M}$ is a finite or countably infinite set of symmetric bounded
linear operators on a real separable Hilbert space $H$ such that:

\begin{itemize}
\item[\emph{(a)}] For every $x\in H$ with $x\neq 0,$ there exists $M\in 
\mathcal{M}$ such that $Mx\neq 0$;

\item[\emph{(b)}] $\sup_{M\in \mathcal{M}}|||M|||<\infty $, where $|||\cdot
|||$ is the operator norm.
\end{itemize}
\end{condition}

Denote $\mathcal{V}\left( \mathcal{M}\right) =\left\{ v:v=(v_{M})_{M\in {%
\mathcal{M}}},~v_{M}\in H\right\} $, so the definition of $\left\Vert \beta
\right\Vert _{\mathcal{M}}$ reads%
\[
\left\Vert \beta \right\Vert _{\mathcal{M}}=\inf \left\{ \sum_{M\in \mathcal{%
M}}\left\Vert v_{M}\right\Vert :v\in \mathcal{V}\left( \mathcal{M}\right) 
\text{ and }\sum_{M\in \mathcal{M}}Mv_{M}=\beta \right\} . 
\]

\begin{theorem}
\label{Theorem norm properties} We have that: 

\begin{itemize}
\item[\emph{(i)}] $\left\Vert \cdot\right\Vert _{\mathcal{M}}$ is positive
homogeneous and subadditive on $\ell _{1}\left( \mathcal{M}\right);$

\item[\emph{(ii)}] $\ell _{1}\left( \mathcal{M}\right) $ is a dense subspace
of $H$. If $\mathcal{M}$ is finite or $H$ is finite dimensional then $\ell
_{1}\left( \mathcal{M}\right) =H;$

\item[\emph{(iii)}] $\left\Vert \cdot\right\Vert _{\mathcal{M}}$ is a norm
on $\ell _{1}\left( \mathcal{M}\right).$
\end{itemize}
\end{theorem}

\begin{proof}
(i) Positive homogeneity of $\left\Vert \cdot\right\Vert _{\mathcal{M}}$ is
clear. For subadditivity let $\beta ,\gamma \in \ell _{1}\left( \mathcal{M}%
\right) $. Let $\epsilon >0$ be arbitrary and choose $w^{\beta },w^{\gamma
}\in \mathcal{V}\left( \mathcal{M}\right) $ such that $\sum_{M\in \mathcal{M}%
}Mw_{M}^{\beta }=\beta $, $\sum_{M\in \mathcal{M}}Mw_{M}^{\gamma }=\gamma $, 
$\sum_{M\in \mathcal{M}}\left\Vert w_{M}^{\beta }\right\Vert \leq \left\Vert
\beta \right\Vert _{\mathcal{M}}+\epsilon $ and $\sum_{M\in \mathcal{M}%
}\left\Vert w_{M}^{\gamma }\right\Vert \leq \left\Vert \gamma \right\Vert _{%
\mathcal{M}}+\epsilon $. Then $w^{\beta }+w^{\gamma }\in \mathcal{V}\left( 
\mathcal{M}\right) $ and $\sum_{M\in \mathcal{M}}M\left( w^{\beta
}+w^{\gamma }\right) _{M}=\beta +\gamma $. Thus $w^{\beta }+w^{\gamma }$ is
in the feasable set for the definition of $\left\Vert \beta +\gamma
\right\Vert _{\mathcal{M}}$ and%
\begin{eqnarray*}
\left\Vert \beta +\gamma \right\Vert _{\mathcal{M}} &=&\inf \left\{
\sum_{M\in \mathcal{M}}\left\Vert v_{M}\right\Vert :v\in \mathcal{V}\left( 
\mathcal{M}\right) \text{ and }\sum_{M\in \mathcal{M}}Mv_{M}=\beta +\gamma
\right\} \\
&\leq &\sum_{M\in \mathcal{M}}\Vert w_{M}^{\beta }+w_{M}^{\gamma }\Vert 
\\
& \leq &
\sum_{M\in \mathcal{M}}\Vert w_{M}^{\beta }\Vert +\sum_{M\in \mathcal{M}%
}\left\Vert w_{M}^{\gamma }\right\Vert \leq \left\Vert \beta \right\Vert _{%
\mathcal{M}}+\left\Vert \gamma \right\Vert _{\mathcal{M}}+2\epsilon .
\end{eqnarray*}%
Since $\epsilon $ was arbitrary subadditivity follows.

(ii) It follows from (i) that $\ell _{1}\left( \mathcal{M}\right) $ is a
linear subspace of $H$. Let $S$ be the set of finite linear combinations of
the form 
\[
S=\left\{ \sum_{i=1}^{K}M_{i}v_{i}:K\in 
\mathbb{N}
\text{, }M_{i}\in \mathcal{M}\text{, }v_{i}\in H\right\} . 
\]%
Then $S$ is a linear subspace of $\ell _{1}\left( \mathcal{M}\right) $ and
contains all vectors of the form $MMv=M^{2}v$ where $M\in \mathcal{M}$ and $%
v\in H$. If $x\in H$ is perpendicular to all of $S$ then for all $M\in 
\mathcal{M}$ we must have $x\perp MMx\iff Mx=0$, which implies $x=0$ by
condition (a). This shows that $S$ and therefore also $\ell _{1}\left( 
\mathcal{M}\right) $ are dense in $H$. The second assertion of (ii) is an
easy consequence of the first.

(iii) Suppose $\beta \in \ell _{1}\left( \mathcal{M}\right) $, $\beta \neq 0$
and $\beta =\sum_{\mathcal{M}}Mv_{M}$ with $v\in \mathcal{V}\left( \mathcal{M%
}\right) $. 
\[
0\leq \left\Vert \beta \right\Vert =\left\Vert \sum_{M\in \mathcal{M}%
}Mv_{M}\right\Vert \leq \sup_{M\in \mathcal{M}}|||M|||\sum_{M\in \mathcal{M}%
}\left\Vert v_{M}\right\Vert .
\]%
Taking the infimum on the right hand side we obtain that 
\[
\left\Vert \beta \right\Vert _{\mathcal{M}}\geq \frac{\left\Vert \beta
\right\Vert }{\sup\limits_{{M\in {\mathcal{M}}}}|||M|||}>0,
\]%
where condition (b) was used. Together with (i) this implies that $%
\left\Vert \cdot \right\Vert _{\mathcal{M}}$ is a norm on $\ell _{1}\left( 
\mathcal{M}\right) $.
\end{proof}

From now on we refer to $\ell _{1}\left( \mathcal{M}\right) $ as the normed
linear space with norm $\left\Vert \cdot\right\Vert _{\mathcal{M}}$.

\begin{theorem}
\label{Theorem duality}Let $z\in H$. The linear functional $\beta \mapsto
\left\langle \beta ,z\right\rangle $ is bounded on $\ell _{1}\left( \mathcal{%
M}\right) $ and has norm%
\[
\left\Vert z\right\Vert _{\mathcal{M}\ast }=\sup_{M\in \mathcal{M}%
}\left\Vert Mz\right\Vert . 
\]
\end{theorem}

\begin{proof}
Let $F$ be the dual norm. By definition%
\begin{eqnarray*}
F\left( z\right) &=&\inf \left\{ s:s~\left\Vert \beta \right\Vert _{\mathcal{%
M}}-\left\langle \beta ,z\right\rangle \geq 0,\forall \beta \in H\right\} \\
&=&\inf \left\{ s:\sum_{M\in \mathcal{M}}\left( s\left\Vert v_{M}\right\Vert
-\left\langle Mv_{M},z\right\rangle \right) \geq 0\text{, }\forall v\in 
\mathcal{V}\left( \mathcal{M}\right) \right\} \\ 
&=&\inf \left\{ s:s\left\Vert v\right\Vert -\left\langle Mv,z\right\rangle
\geq 0\text{, }\forall v\in H,\forall M\in \mathcal{M}\right\} \\
&=&\inf \left\{ s:s\geq \left\langle v,Mz\right\rangle \text{, }\forall v\in
H,\left\Vert v\right\Vert =1,\forall M\in \mathcal{M}\right\}  \\ 
&=&\inf \left\{ s:s\geq \left\Vert Mz\right\Vert \text{, }\forall M\in 
\mathcal{M}\right\}  \\ 
&=&\sup_{M\in \mathcal{M}}\left\Vert Mz\right\Vert =\left\Vert z\right\Vert
_{\mathcal{M}\ast }.
\end{eqnarray*}
\end{proof}

\begin{proposition}
\label{Proposition simplification} If the ranges of the members of $%
\mathcal{M}$ are mutually orthogonal then for $\beta \in \ell _{1}\left( 
\mathcal{M}\right) $%
\[
\left\Vert \beta \right\Vert _{\mathcal{M}}=\sum_{M\in \mathcal{M}%
}\left\Vert M^{+}\beta \right\Vert ,
\]%
where $M^{+}$ is the pseudoinverse of $M$.
\end{proposition}

\begin{proof}
The ranges of the members of $\mathcal{M}$ provide an orthogonal
decomposition of $H$, so 
\[
\beta =\sum_{M\in \mathcal{M}}M\left( M^{+}\beta \right) ,
\]%
where we used the fact that $MM^{+}$ is the orthogonal projection onto the
range of $M$. Taking $v_{M}=M^{+}\beta $ this implies that $\left\Vert \beta
\right\Vert _{\mathcal{M}}\leq \sum_{M\in \mathcal{M}}\left\Vert M^{+}\beta
\right\Vert $. On the other hand, if $\beta =\sum_{N\in \mathcal{M}}Nv_{N}$,
then, applying $M^{+}$ to this identity we see that $M^{+}Mv_{M}=M^{+}\beta $
for all $M$, so%
\[
\sum_{M\in \mathcal{M}}\left\Vert v_{M}\right\Vert \geq \sum_{M\in \mathcal{M%
}}\left\Vert M^{+}Mv_{M}\right\Vert =\sum_{M\in \mathcal{M}}\left\Vert
M^{+}\beta \right\Vert ,
\]%
which shows the reverse inequality. \bigskip 
\end{proof}

\subsection{Data and Distribution Dependent Bounds\label{Section Proofs of
bounds}}

We use the bounded difference inequality to derive a concentration
inequality for linearly transformed random vectors.

\begin{lemma}
\label{Lemma McDiarmid Application} Let $\mathbf{\epsilon }=\left( \epsilon
_{1},\dots ,\epsilon _{n}\right) $ be a vector of independent real random
variables with $-1\leq \epsilon _{i}\leq 1$, and $\mathbf{\epsilon }^{\prime
}$ iid to $\mathbf{\epsilon }$. Suppose that $M$ is a linear transformation $%
M:%
\mathbb{R}
^{n}\rightarrow H$.

\begin{itemize}
\item[\emph{(i)}] Then for $t>0$ we have%
\[
\Pr \left\{ \left\Vert M\mathbf{\epsilon }\right\Vert \geq {{\mathbb{E}}}%
\left\Vert M\mathbf{\epsilon }^{\prime }\right\Vert +t\right\} \leq \exp
\left( \frac{-t^{2}}{2\left\Vert M\right\Vert _{HS}^{2}}\right) . 
\]

\item[\emph{(ii)}] If $\mathbf{\epsilon }$ is orthonormal (satisfying ${{%
\mathbb{E}}}\epsilon _{i}\epsilon _{j}=\delta _{ij}$), then 
\begin{equation}
{{\mathbb{E}}}\left\Vert M\mathbf{\epsilon }\right\Vert \leq \left\Vert
M\right\Vert _{HS}.  \label{eq Norm Less Than Frob}
\end{equation}%
and, for every $r>0$,%
\[
\Pr \left\{ \left\Vert M\mathbf{\epsilon }\right\Vert >t\right\} \leq
e^{1/r}\exp \left( \frac{-t^{2}}{\left( 2+r\right) \left\Vert M\right\Vert
_{HS}^{2}}\right) . 
\]
\end{itemize}
\end{lemma}

\begin{proof}
(i) Define $F:\left[ -1,1\right] ^{n}\rightarrow 
\mathbb{R}
$ by $F\left( \mathbf{x}\right) =\left\Vert M\mathbf{x}\right\Vert $. By the
triangle inequality%
\begin{eqnarray*}
&&\sum_{k=1}^{n}\sup_{y_{1},y_{2}\in \left[ -1,1\right] \text{, }\mathbf{x}%
\in \left[ -1,1\right] ^{n}}\left( F\left( \mathbf{x}_{k\leftarrow
y_{1}}\right) -F\left( \mathbf{x}_{k\leftarrow y_{2}}\right) \right) ^{2} \\
&\leq &\sum_{k=1}^{n}\sup_{y_{1},y_{2}\in \left[ -1,1\right] \text{, }%
\mathbf{x}\in \left[ -1,1\right] ^{n}}\left\Vert M\left( \mathbf{x}%
_{k\leftarrow y_{1}}-\mathbf{x}_{k\leftarrow y_{2}}\right) \right\Vert ^{2}
\\
&=&\sum_{k=1}^{n}\sup_{y_{1},y_{2}\in \left[ -1,1\right] }\left(
y_{1}-y_{2}\right) ^{2}\left\Vert Me_{k}\right\Vert ^{2} \\
&\leq &4\left\Vert M\right\Vert _{HS}^{2}
\end{eqnarray*}%
The result now follows from the bounded difference inequality (Theorem \ref%
{Theorem Bded Difference}).

(ii) If $\mathbf{\epsilon }$ is orthonormal then it follows from Jensen's
inequality that 
\[
{{\mathbb{E}}}\left\Vert M\mathbf{\epsilon }\right\Vert \leq \left( {{%
\mathbb{E}}}\left\Vert \sum_{i=1}^{n}\epsilon _{i}Me_{i}\right\Vert
^{2}\right) ^{1/2}=\left( \sum_{i}\left\Vert Me_{i}\right\Vert ^{2}\right)
^{1/2}=\left\Vert M\right\Vert _{HS}. 
\]%
For the second assertion of (ii) first note that from calculus we get $%
\left( t-1\right) ^{2}/2-t^{2}/\left( 2+r\right) \geq -1/r$ for all $t\in 
\mathbb{R}
$. This implies that%
\begin{equation}
e^{-\left( t-1\right) ^{2}/2}\leq e^{1/r}e^{-t^{2}/\left( 2+r\right) }.
\label{eq expo inequality}
\end{equation}%
Since $1/r\geq 1/\left( 2+r\right) $ the inequality to be proved is trivial
for $t\leq \left\Vert M\right\Vert _{HS}$. If $t>\left\Vert M\right\Vert
_{HS}$ then, using ${{\mathbb{E}}}\left\Vert M\mathbf{\epsilon }\right\Vert
\leq \left\Vert M\right\Vert _{HS}$, we have $t-E\left\Vert M\mathbf{%
\epsilon }\right\Vert \geq t-\left\Vert M\right\Vert _{HS}>0$, so by part
(i) and (\ref{eq expo inequality}) we obtain 
\begin{eqnarray*}
\Pr \left\{ \left\Vert M\mathbf{\epsilon }\right\Vert \geq t\right\} &=&\Pr
\left\{ \left\Vert M\mathbf{\epsilon }\right\Vert \geq E\left\Vert M\mathbf{%
\epsilon }\right\Vert +\left( t-E\left\Vert M\mathbf{\epsilon }\right\Vert
\right) \right\} \\
&\leq &\exp \left( \frac{-\left( t-E\left\Vert M\mathbf{\epsilon }%
\right\Vert \right) ^{2}}{2\left\Vert M\right\Vert _{HS}^{2}}\right) \leq
\exp \left( \frac{-\left( t-\left\Vert M\right\Vert _{HS}\right) ^{2}}{%
2\left\Vert M\right\Vert _{HS}^{2}}\right) \\
&=&\exp \left( \frac{-\left( t/\left\Vert M\right\Vert _{HS}-1\right) ^{2}}{2%
}\right) \leq e^{1/r}e^{-\left( t/\left\Vert M\right\Vert _{HS}\right)
^{2}/\left( 2+r\right) } \\
&=&e^{1/r}\exp \left( \frac{-t^{2}}{\left( 2+r\right) \left\Vert
M\right\Vert _{HS}^{2}}\right) .
\end{eqnarray*}%
\bigskip
\end{proof}

We now use integration by parts, a union bound and the above concentration
inequality to derive a bound on the expectation of the supremum of the norms 
$\left\Vert M\mathbf{\epsilon }\right\Vert $. This is the essential step in
the proof of Theorem \ref{Theorem Main}. It is by no means a new technique,
in fact it appears many times in the book by Ledoux and Talagrand \cite%
{Ledoux 1991}, but compared to the combinatorial approach in \cite{Cortes}
it seems more suited to the study of the problem at hand, and gives insights
into the fine structure of the logarithmic factor appearing in bounds for
Lasso-like methods.

\begin{lemma}
\label{Lemma Key lemma}Let $\mathcal{M}$ be a finite or countably infinite
set of linear transformations $M:%
\mathbb{R}
^{n}\rightarrow H$ and $\mathbf{\epsilon }=\left( \epsilon _{1},\dots
,\epsilon _{n}\right) $ a vector of orthonormal random variables (satisfying 
${{\mathbb{E}}}\epsilon _{i}\epsilon _{j}=\delta _{ij}$) with values in $%
\left[ -1,1\right] $. Then%
\[
{{\mathbb{E}}}\sup_{M\in \mathcal{M}}\left\Vert M\mathbf{\epsilon }%
\right\Vert \leq \sqrt{2}\sup_{M\in \mathcal{M}}\left\Vert M\right\Vert
_{HS}\left( 2+\sqrt{\ln \frac{\sum_{M\in \mathcal{M}}\left\Vert M\right\Vert
_{HS}^{2}}{\sup_{M\in \mathcal{M}}\left\Vert M\right\Vert _{HS}^{2}}}\right)
. 
\]
\end{lemma}

\begin{proof}
To lighten notation we abbreviate $\mathcal{M}_{\infty }:=\sup_{M\in 
\mathcal{M}}\left\Vert M\right\Vert _{HS}$ below. We now use integration by
parts%
\begin{eqnarray*}
{{\mathbb{E}}}\sup_{M\in \mathcal{M}}\left\Vert M\mathbf{\epsilon }%
\right\Vert &=&\int_{0}^{\infty }\Pr \left\{ \sup_{M\in \mathcal{M}%
}\left\Vert M\mathbf{\epsilon }\right\Vert >t\right\} dt \\
&\leq &\mathcal{M}_{\infty }+\delta +\int_{\mathcal{M}_{\infty }+\delta
}^{\infty }\Pr \left\{ \sup_{M\in \mathcal{M}}\left\Vert M\mathbf{\epsilon }%
\right\Vert >t\right\} dt \\
&\leq &\mathcal{M}_{\infty }+\delta +\sum_{M\in \mathcal{M}}\int_{\mathcal{M}%
_{\infty }+\delta }^{\infty }\Pr \left\{ \left\Vert M\mathbf{\epsilon }%
\right\Vert >t\right\} dt,
\end{eqnarray*}%
where we have introduced a parameter $\delta \geq 0$. The first inequality
above follows from the fact that probabilities never exceed $1$, and the
second from a union bound. Now for any $M\in \mathcal{M}$ we can make a
change of variables and use (\ref{eq Norm Less Than Frob}), which gives ${{%
\mathbb{E}}}\left\Vert M\mathbf{\epsilon }\right\Vert \leq \left\Vert
M\right\Vert _{HS}\leq \mathcal{M}_{\infty }$, so that%
\begin{eqnarray*}
\int_{\mathcal{M}_{\infty }+\delta }^{\infty }\Pr \left\{ \left\Vert M%
\mathbf{\epsilon }\right\Vert >t\right\} dt &\leq &\int_{\delta }^{\infty
}\Pr \left\{ \left\Vert M\mathbf{\epsilon }\right\Vert >{{\mathbb{E}}}%
\left\Vert M\mathbf{\epsilon }\right\Vert +t\right\} dt \\
&\leq &\int_{\delta }^{\infty }\exp \left( \frac{-t^{2}}{2\left\Vert
M\right\Vert _{HS}^{2}}\right) dt \\
&\leq &\frac{\left\Vert M\right\Vert _{HS}^{2}}{\delta }\exp \left( \frac{%
-\delta ^{2}}{2\left\Vert M\right\Vert _{HS}^{2}}\right) ,
\end{eqnarray*}%
where the second inequality follows from Lemma \ref{Lemma McDiarmid
Application} (i), and the third from Lemma \ref{Lemma Normal approximation}.
Substitution in the previous chain of inequalities and using Hoelder's
inequality (in the $\ell _{1}/\ell _{\infty }$-version) give%
\begin{equation}
{{\mathbb{E}}}\sup_{M\in \mathcal{M}}\left\Vert M\mathbf{\epsilon }%
\right\Vert \leq \mathcal{M}_{\infty }+\delta +\frac{1}{\delta }\left(
\sum_{M\in \mathcal{M}}\left\Vert M\right\Vert _{HS}^{2}\right) \exp \left( 
\frac{-\delta ^{2}}{2\mathcal{M}_{\infty }^{2}}\right) .
\label{eq use Hoelder}
\end{equation}%
We now set 
\[
\delta =\sqrt{2\ln \left( e\frac{\sum_{M\in \mathcal{M}}\left\Vert
M\right\Vert _{HS}^{2}}{\mathcal{M}_{\infty }^{2}}\right) }\mathcal{M}%
_{\infty }. 
\]%
Then $\delta \geq 0$ as required. The substitution makes the last term in (%
\ref{eq use Hoelder}) smaller than $\mathcal{M}_{\infty }/\left( e\sqrt{2}%
\right) $, and since $1+1/\left( e\sqrt{2}\right) <\sqrt{2}$, we obtain%
\[
{{\mathbb{E}}}\sup_{M\in \mathcal{M}}\left\Vert M\mathbf{\epsilon }%
\right\Vert \leq \sqrt{2}\mathcal{M}_{\infty }\left( 1+\sqrt{\ln \left( 
\frac{e\sum_{M\in \mathcal{M}}\left\Vert M\right\Vert _{HS}^{2}}{\mathcal{M}%
_{\infty }^{2}}\right) }\right) . 
\]%
Finally we use $\sqrt{\ln es}\leq 1+\sqrt{\ln s}$ for $s\geq 1$.\bigskip
\end{proof}

\begin{proof}[Proof of Theorem \protect\ref{Theorem Main}]
Let $\mathbf{\epsilon }=\left( \epsilon _{1},\dots ,\epsilon _{n}\right) $
be a vector of iid Rademacher variables. For $M\in \mathcal{M}$ we use $M%
\mathbf{x}$ to denote the linear transformation $M\mathbf{x}:%
\mathbb{R}
^{n}\rightarrow H$ given by $\left( M\mathbf{x}\right) \mathbf{y}%
=\sum_{i}\left( Mx_{i}\right) y_{i}$. We have%
\[
\mathcal{R}_{\mathcal{M}}\left( \mathbf{x}\right) =\frac{2}{n}{{\mathbb{E}}}%
\sup_{\beta :\left\Vert \beta \right\Vert _{\mathcal{M}}\leq 1}\left\langle
\beta ,\sum_{i=1}^{n}\epsilon _{i}x_{i}\right\rangle \leq \frac{2}{n}{{%
\mathbb{E}}}\left\Vert \sum_{i=1}^{n}\epsilon _{i}x_{i}\right\Vert _{%
\mathcal{M\ast }}=\frac{2}{n}{{\mathbb{E}}}\sup_{M\in \mathcal{M}}\left\Vert
M\mathbf{x\epsilon }\right\Vert . 
\]%
Applying Lemma \ref{Lemma Key lemma} to the set of transformations $\mathcal{%
M}\mathbf{x}=\left\{ M\mathbf{x}:M\in \mathcal{M}\right\} $ gives%
\[
\mathcal{R}_{\mathcal{M}}\left( \mathbf{x}\right) \leq \frac{%
2^{3/2}\sup_{M\in \mathcal{M}}\left\Vert M\mathbf{x}\right\Vert _{HS}}{n}%
\left( 2+\sqrt{\ln \frac{\sum_{M\in \mathcal{M}}\left\Vert M\mathbf{x}%
\right\Vert _{HS}^{2}}{\sup_{M\in \mathcal{M}}\left\Vert M\mathbf{x}%
\right\Vert _{HS}^{2}}}\right) . 
\]%
Substitution of $\left\Vert M\mathbf{x}\right\Vert
_{HS}^{2}=\sum_{i}\left\Vert Mx_{i}\right\Vert ^{2}$ gives the first
inequality of Theorem \ref{Theorem Main} and%
\[
\sup_{M\in \mathcal{M}}\left\Vert M\mathbf{x}\right\Vert _{HS}^{2}\leq
\sum_{i}\sup_{M\in \mathcal{M}}\left\Vert Mx_{i}\right\Vert
^{2}=\sum_{i}\left\Vert x_{i}\right\Vert _{\ast \mathcal{M}}^{2} 
\]%
gives the second inequality.
\end{proof}

\bigskip

\begin{proof}[Proof of Corollary \protect\ref{Theorem L2 bound}]
From calculus we find that $t\ln t\geq -1/e$ for all $t>0$. For $A,B>0$ and $%
n\in 
\mathbb{N}
$ this implies that 
\begin{equation}
A\ln \frac{B}{A}=n\left[ \left( A/n\right) \ln \left( B/n\right) -\left(
A/n\right) \ln \left( A/n\right) \right] \leq A\ln \left( B/n\right) +n/e.
\label{eq logbound}
\end{equation}%
Now multiply out the first inequality of Theorem \ref{Theorem Main} and use (%
\ref{eq logbound}) with 
\[
A=\sup_{M\in \mathcal{M}}\sum_{i}\left\Vert Mx_{i}\right\Vert ^{2}\text{ and 
}B=\sum_{M\in \mathcal{M}}\sum_{i}\left\Vert Mx_{i}\right\Vert ^{2}. 
\]%
Finally use $\sqrt{a+b}\leq \sqrt{a}+\sqrt{b}$ for $a,b>0$ and the fact that 
$2^{3/2}/\sqrt{e}\leq 2$.\bigskip
\end{proof}


\section{Extension to the $\ell _{q}\left( \mathcal{M}\right)$ Case}
\label{sec:Lp}
There is a rather obvious extension of our framework, which should be
mentioned for completeness: Let $q$ and $p$ be conjugate exponents (i.e. $%
1/q+1/p=1$) and define 
\[
\left\Vert \beta \right\Vert _{\mathcal{M}_{q}}=\inf \left\{ \left(
\sum_{M\in \mathcal{M}}\left\Vert v_{M}\right\Vert ^{q}\right)
^{1/q}:v_{M}\in H\text{ and }\sum_{M\in \mathcal{M}}Mv_{M}=\beta \right\} , 
\]%
in analogy to (\ref{def Omega}). Then $\left\Vert \beta \right\Vert _{%
\mathcal{M}_{q}}$ is a norm and the dual norm is given by%
\[
\left\Vert z\right\Vert _{\mathcal{M}_{q}\ast }=\left( \sum_{M\in \mathcal{M}%
}\left\Vert Mz\right\Vert ^{p}\right) ^{1/p}. 
\]%
The proof of these facts is omitted in this version of the paper. In the
following we give a result, which can be applied to cases analogous to those
in Section \ref{sec:examples}, where it recovers existing results up to
constant multiplicative factors.

\begin{theorem}
\label{Theorem L_q bound}Let $\mathbf{x}$ be a sample and $\mathcal{R}_{%
\mathcal{M}_{q}}\left( \mathbf{x}\right) $ the empirical Rademacher
complexity of the class of linear functions parameterized by $\beta $ with $%
\left\Vert \beta \right\Vert _{\mathcal{M}_{q}}\leq 1$. Then for $1<q\leq 2$%
\[
\mathcal{R}_{\mathcal{M}_{q}}\left( \mathbf{x}\right) \leq \frac{2^{3/2}}{n}%
\sqrt{\pi p\sum_{i}\left\Vert x_{i}\right\Vert _{\mathcal{M}_{q}\mathcal{%
\ast }}^{2}}. 
\]
\end{theorem}

The proof is analogous to the proof of Theorem \ref{Theorem Main}, but
somewhat more straightforward.

\begin{lemma}
\label{Lemma L_p}Let $\mathcal{M}$ be a finite or countably infinite set of
linear transformations $M:%
\mathbb{R}
^{n}\rightarrow H$ and $\mathbf{\epsilon }=\left( \epsilon _{1},\dots
,\epsilon _{n}\right) $ a vector of orthonormal random variables (satisfying 
${{\mathbb{E}}}\epsilon _{i}\epsilon _{j}=\delta _{ij}$) with values in $%
\left[ -1,1\right] $. Then for $p\geq 2$%
\[
{{\mathbb{E}}}\left[ \left( \sum_{M\in \mathcal{M}}\left\Vert M\mathbf{%
\epsilon }\right\Vert ^{p}\right) ^{1/p}\right] \leq \sqrt{2\pi p}\left(
\sum_{M\in \mathcal{M}}\left\Vert M\right\Vert _{HS}^{p}\right) ^{1/p}. 
\]
\end{lemma}

\begin{proof}
First note that by standard results on the absolute moments of the normal
distribution%
\[
\int_{0}^{\infty }t^{p-1}\exp \left( \frac{-t^{2}}{2}\right) dt\leq \sqrt{%
\frac{\pi }{2}}\left( p-2\right) !!\leq \sqrt{\frac{\pi }{2}}\left( 1\cdot
3\cdot ...\cdot p-2\right) \leq \sqrt{\frac{\pi }{2}}p^{p/2-1}, 
\]%
so%
\begin{equation}
\left( p\int_{0}^{\infty }t^{p-1}\exp \left( \frac{-t^{2}}{2}\right)
dt\right) ^{1/p}\leq \sqrt{\frac{\pi }{2}}^{1/p}p^{1/2}\leq \sqrt{\frac{\pi p%
}{2}}.  \label{eq absolute moments}
\end{equation}%
Jensen's inequality and integration by parts give%
\begin{eqnarray*}
{{\mathbb{E}}}\left( \sum_{M\in \mathcal{M}}\left\Vert M\mathbf{\epsilon }%
\right\Vert ^{p}\right) ^{1/p} &\leq &\left( \sum_{M\in \mathcal{M}}{{%
\mathbb{E}}}\left\Vert M\mathbf{\epsilon }\right\Vert ^{p}\right)
^{1/p}=\left( \sum_{M\in \mathcal{M}}p\int_{0}^{\infty }\Pr \left\{
\left\Vert M\mathbf{\epsilon }\right\Vert >t\right\} t^{p-1}dt\right) ^{1/p}
\\
&\leq &\left( 2p\sum_{M\in \mathcal{M}}\int_{0}^{\infty }t^{p-1}\exp \left( 
\frac{-t^{2}}{4\left\Vert M\right\Vert _{HS}^{2}}\right) dt\right) ^{1/p},
\end{eqnarray*}%
where Lemma \ref{Lemma McDiarmid Application} (ii) was used in the last step
with $r=2$. A change of variables $t\rightarrow t/\left( \sqrt{2}\left\Vert
M\right\Vert _{HS}\right) $ gives%
\begin{eqnarray*}
{{\mathbb{E}}}\left( \sum_{M\in \mathcal{M}}\left\Vert M\mathbf{\epsilon }%
\right\Vert ^{p}\right) ^{1/p} &\leq &\left( 2p\int_{0}^{\infty }t^{p-1}\exp
\left( \frac{-t^{2}}{2}\right) dt\sum_{M\in \mathcal{M}}2^{p/2}\left\Vert
M\right\Vert _{HS}^{p}\right) ^{1/p} \\
&\leq &2^{1/p+1/2}\sqrt{\frac{\pi p}{2}}\left( \sum_{M\in \mathcal{M}%
}\left\Vert M\right\Vert _{HS}^{p}\right) ^{1/p},
\end{eqnarray*}%
where we use (\ref{eq absolute moments}) in the last inequality.\bigskip
\end{proof}

\begin{proof}[Proof of Theorem \protect\ref{Theorem L_q bound}]
As in the proof of Theorem \ref{Theorem Main} we proceed using duality and
apply Lemma \ref{Lemma L_p} to the set of transformations $\mathcal{M}%
\mathbf{x}=\left\{ M\mathbf{x}:M\in \mathcal{M}\right\} $.%
\begin{eqnarray*}
\mathcal{R}_{\mathcal{M}_{q}}\left( \mathbf{x}\right) &\leq &\frac{2}{n}{{%
\mathbb{E}}}\left\Vert \sum \epsilon _{i}x_{i}\right\Vert _{\mathcal{M}%
_{q}\ast }=\frac{2}{n}{{\mathbb{E}}}\left[ \left( \sum_{M\in \mathcal{M}%
}\left\Vert M\mathbf{x\epsilon }\right\Vert ^{p}\right) ^{1/p}\right] \\
&\leq &\frac{2}{n}\sqrt{2\pi p}\left( \sum_{M\in \mathcal{M}}\left\Vert M%
\mathbf{x}\right\Vert _{HS}^{p}\right) ^{1/p}=\frac{2}{n}\sqrt{2\pi p\left(
\sum_{M\in \mathcal{M}}\left( \sum_{i}\left\Vert Mx_{i}\right\Vert
^{2}\right) ^{p/2}\right) ^{2/p}} \\
&\leq &\frac{2}{n}\sqrt{2\pi p\sum_{i}\left( \sum_{M\in \mathcal{M}}\left(
\left\Vert Mx_{i}\right\Vert ^{2}\right) ^{p/2}\right) ^{2/p}}=\frac{2^{3/2}%
}{n}\sqrt{\pi p\sum_{i}\left\Vert x_{i}\right\Vert _{\mathcal{M}_{p}\mathcal{%
\ast }}^{2}},
\end{eqnarray*}%
where the last inequality is just the triangle inequality in $\ell _{p/2}$.
\end{proof}

\section{Conclusion and Future Work}

We have presented a bound on the Rademacher average for linear function
classes described by infimum convolution norms which are associated with a
class of bounded linear operators on a Hilbert space. We highlighted the generality of the approach
and its dimension independent features.

When the bound is applied to specific cases ($\ell _{2}$, $\ell _{1}$, mixed 
$\ell _{1}/\ell _{2}$ norms) it recovers existing bounds (up to small
changes in the constants). The bound is however more general and allows for
the possibility to remove the \textquotedblleft $\log d$\textquotedblright\
factor which appears in previous bounds. Specifically, we have shown that
the bound can be applied in infinite dimensional settings, provided that the
moment condition \eqref{eq second moment condition} is satisfied. 
We have also applied the bound to multiple kernel learning. 
While in the standard case the bound is only slightly worse in the
constants, the bound is potentially smaller and applies to the more general
case in which there is a countable set of kernels, provided the expectation
of the sum of the kernels is bounded.

An interesting question is whether the bound presented is tight. As noted in 
\cite{Cortes} the ``$\log d$'' is unavoidable. This result immediately
implies that our bound is also tight, since we may choose $R^2=d$ in
equation \eqref{eq second moment condition}.

A potential future direction of research is the application of our
results in the context of sparsity oracle inequalities. In particular, it would be interesting to modify the analysis in 
\cite{lounici}, in order to derive dimension independent bounds. Another interesting scenario is 
the combination of our analysis with metric entropy.

\subsection*{Acknowledgements}
We wish to thank Andreas Argyriou, Bharath Sriperumbudur, Alexandre Tsybakov and Sara van de Geer, for useful comments.
Part of this work was supported by EPSRC Grant EP/H027203/1 and Royal Society International Joint Project Grant 2012/R2.

\end{document}